\documentclass[10pt,twocolumn,letterpaper]{article}

\usepackage{cvpr}              

\usepackage{amsmath}
\usepackage{amssymb}
\usepackage{mathtools}
\usepackage{amsthm}
\usepackage{tipa}
\usepackage{bm}
\usepackage{bbold}
\usepackage{dsfont}
\usepackage{algorithm} 
\usepackage{algpseudocode} 
\usepackage{appendix}
\usepackage{titletoc}

\theoremstyle{plain}
\newtheorem{theorem}{Theorem}[section]

\newtheorem{lemma}[theorem]{Lemma}

\theoremstyle{definition}

\newtheorem{assumption}[theorem]{Assumption}
\theoremstyle{remark}

\definecolor{cvprblue}{rgb}{0.21,0.49,0.74}
\usepackage[pagebackref,breaklinks,colorlinks,allcolors=cvprblue]{hyperref}
\usepackage{multirow}
\usepackage{booktabs}
\usepackage{multicol}

\def\paperID{12289} 
\def\confName{CVPR} 
\def\confYear{2025} 

\title{NLPrompt: Noise-Label Prompt Learning for Vision-Language Models}  

\author{Bikang Pan$^{1,\dagger}$\quad  Qun Li$^{1,\dagger}$ \quad Xiaoying Tang$^2$\quad Wei Huang$^3$\quad Zhen Fang$^4$\quad Feng Liu$^5$\quad \\Jingya Wang$^1$\quad Jingyi Yu$^1$\quad Ye Shi$^{1,}$\thanks{Corresponding author. $\dagger$ Equal contribution.} \\
$^1$ShanghaiTech University, Shanghai, China\\
$^2$The Chinese University of Hong Kong, Shenzhen, China\\
$^3$RIKEN Center for Advanced Intelligence Project, Japan\\
$^4$University of Technology Sydney, Australia \ \ 
$^5$University of Melbourne, Australia \\
{\tt\small \{panbk2023,liqun2024,wangjingya,yujingyi,shiye\}@shanghaitech.edu.cn,} \\{\tt\small tangxiaoying@cuhk.edu.cn,wei.huang.vr@riken.jp,}\\
{\tt\small zhen.fang@uts.edu.au,feng.liu1@unimelb.edu.au}\\
{\small {\url{https://github.com/qunovo/NLPrompt}}}}

\begin{document}
\maketitle

\begin{abstract}
The emergence of vision-language foundation models, such as CLIP, has revolutionized image-text representation, enabling a broad range of applications via prompt learning. Despite its promise, real-world datasets often contain noisy labels that can degrade prompt learning performance. In this paper, we demonstrate that using mean absolute error (MAE) loss in prompt learning, named PromptMAE, significantly enhances robustness against noisy labels while maintaining high accuracy. Though MAE is straightforward and recognized for its robustness, it is rarely used in noisy-label learning due to its slow convergence and poor performance outside prompt learning scenarios. To elucidate the robustness of PromptMAE, we leverage feature learning theory to show that MAE can suppress the influence of noisy samples, thereby improving the signal-to-noise ratio and enhancing overall robustness. Additionally, we introduce PromptOT, a prompt-based optimal transport data purification method to enhance the robustness further. PromptOT employs text features in vision-language models as prototypes to construct an optimal transportation matrix. This matrix effectively partitions datasets into clean and noisy subsets, allowing for the application of cross-entropy loss to the clean subset and MAE loss to the noisy subset. Our Noise-Label Prompt Learning method, named NLPrompt, offers a simple and efficient approach that leverages the expressive representations and precise alignment capabilities of vision-language models for robust prompt learning. We validate NLPrompt through extensive experiments across various noise settings, demonstrating significant performance improvements. 
\end{abstract} 

\section{Introduction} 
The advent of vision-language foundation models, such as CLIP \cite{CLIP}, has revolutionized how images and their textual descriptions are represented, providing a unified perspective for both modalities. In these models, images are typically aligned with sentences like ``A photo of a $\langle$CLS$\rangle$", thereby facilitating the efficient handling of various tasks. Given the sensitivity of handcrafted text in descriptions, prompt learning has emerged as a crucial method for fine-tuning these vision-language models. Prompt learning involves updating a learnable text prompt through back-propagation \cite{CoOp, zhouConditionalPromptLearning2022, chen2022plot, li2024global, cui2024harmonizing}, offering a lightweight solution due to the relatively small number of parameters involved, often just several thousand. This adaptability ensures rapid tuning for specific tasks. 

Nevertheless, real-world applications often face the challenge of dealing with noisy labels in annotated datasets, necessitating robust learning solutions. Prior work \cite{wuWhyPromptTuning2023} illustrates that prompt tuning is more resilient to noisy labels compared to other fine-tuning paradigms such as adapter tuning. Despite this, prompt tuning remains vulnerable to overfitting when trained with cross-entropy loss under noisy conditions. Therefore, enhancing the robustness of prompt tuning in noisy environments remains a crucial issue. 

In the realm of noisy label learning, mean absolute error (MAE) has been identified as a robust loss function within the traditional training paradigm \cite{ghoshRobustLossFunctions2017}. However, MAE often suffers from slow convergence and poor performance during training, making it seldom employed as a classification loss in noise-label learning. Nevertheless, our investigation reveals an interesting phenomenon: \textit{employing MAE loss in Prompt learning (PromptMAE) notably enhances robustness while maintaining high accuracy compared to traditional cross-entropy loss}. As demonstrated in Figure \ref{fig:cevsmae}, MAE exhibits strong accuracy and fast convergence even in the presence of substantial noise. 

\begin{figure*}[htbp]
\centering
\includegraphics[width=\linewidth]{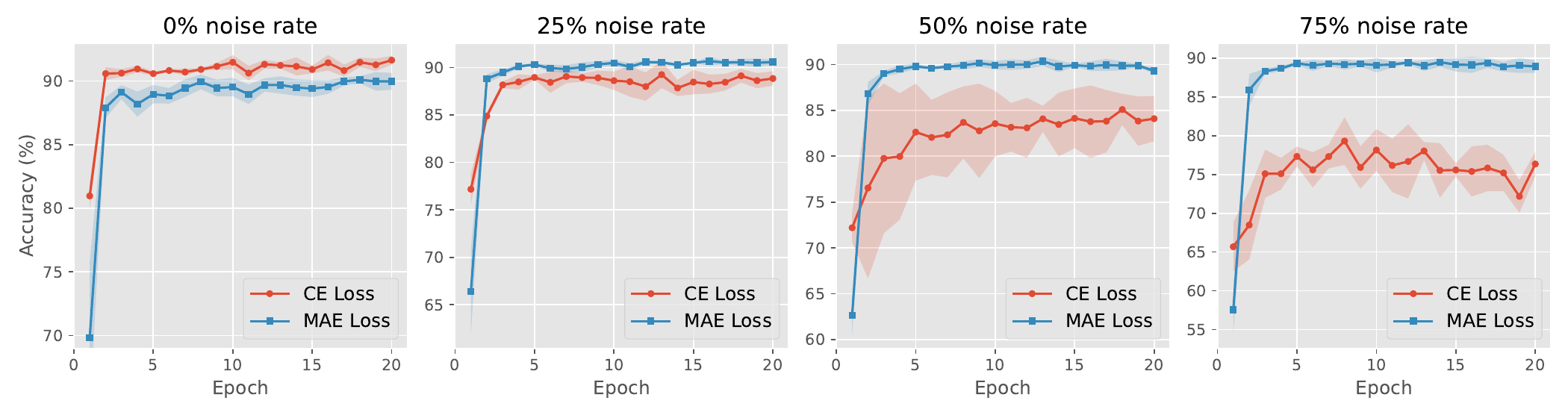}
 \caption{The performance of training with MAE loss and CE loss in prompt learning on Caltech101 dataset.} 
 \label{fig:cevsmae}
\end{figure*}

To elucidate the robustness of PromptMAE, we leverage feature learning theory \cite{allen-zhuUnderstandingEnsembleKnowledge2022, caoBenignOverfittingTwolayer2022, panFederatedLearningVisionLanguage2024}, which categorizes latent representations into task-relevant and task-irrelevant components. By analyzing the optimization dynamics of these features with gradient-descent-based training, we can gain valuable insights into convergence and generalization. To this end, we find that robust prompt learning is achieved when task-relevant features dominate. Our analysis indicates that PromptMAE can suppress the influence of noisy samples, thereby enhancing robustness in prompt learning for vision-language models. 

A standard approach in noisy label learning is the employment of sample selection techniques \cite{chengLearningInstancedependentLabel2020, liDividemixLearningNoisy2020, hanCoteachingRobustTraining2018, patelAdaptiveSampleSelection2023, fengOtfilterOptimalTransport2023} to clean the dataset and thus improve performance under noisy conditions. For example, optimal transport (OT)-based sample selection methods \cite{fengOtfilterOptimalTransport2023, xiaOtCleanerLabel2022} utilize randomly initialized prototypes to compute the optimal transportation matrix from image features to these prototypes, considering the similarity between features and prototypes as a cost matrix. As these methods were not originally designed for prompt learning, their direct applicability may be limited. We aim to harness the inherent alignment in vision-language foundation models to refine the data purification process. 

In this paper, we introduce PromptOT, a prompt-based optimal transport data purification method, designed to enhance the robustness of prompt learning in vision-language foundation models. PromptOT leverages the text features as prototypes for the transportation matrix, facilitating robust prompt learning by partitioning the dataset into clean and noisy subsets. Recognizing that cross-entropy (CE) loss generally outperforms MAE on clean datasets, we apply MAE loss to train the noisy subset and cross-entropy loss to train the clean subset. This dual strategy, supported by PromptOT purification, harmonizes the strengths of both MAE and CE loss under varying noisy conditions. Our comprehensive method, named NLPrompt, leverages the expressive representation and alignment capabilities of vision-language models, offering a simple and efficient solution for robust prompt learning in the presence of noisy labels. In summary, our contributions are threefold: 
\\

\begin{itemize}
    \item We discover that a simple MAE loss significantly improves the robustness of prompt learning on noisy datasets. Utilizing feature learning theory, we theoretically demonstrate how PromptMAE reduces the impact of noisy samples, enhancing overall robustness. 
    \\
    
    \item We introduce NLPrompt, a robust prompt learning method that uses a simple MAE loss with PromptOT-based data purification to handle noisy labels. NLPrompt efficiently exploits the expressive representation and precise alignment capabilities of vision-language foundation models for robust prompt learning. 
    \\
    
    \item We validate the effectiveness of NLPrompt through extensive experiments across datasets with varied noise conditions, consistently showing significant performance improvements. 
\end{itemize} 

\section{Related Work}
\subsection{Prompt Learning in Vision-Language Models}
Prompt learning, which began in natural language processing, has now extended into the realm of vision-language models. A notable example is the CLIP model \cite{CLIP}, which initially relied on hand-crafted prompts. Recent advancements, however, have shifted focus towards learning prompts in a continuous embedding space. Innovations such as CoOp \cite{CoOp} have enhanced the CLIP model by integrating continuous prompt vectors, fostering a wave of research dedicated to optimizing prompt learning and paving the way for further exploration. In addition to CoOp, CoCoOp \cite{zhouConditionalPromptLearning2022} utilizes a neural network to generate input-specific context tokens that adapt the prompts based on each image, thereby improving generalization to unseen classes. ProGrad \cite{zhuPromptalignedGradientPrompt2023} regularizes the soft prompt updates by aligning their gradients with the general knowledge provided by the original prompt. MaPLe \cite{khattakMapleMultimodalPrompt2023} introduces branch-aware hierarchical prompts that address both language and vision branches. TPT (Test-Time Prompt Tuning) \cite{shuTesttimePromptTuning2022} explores prompt tuning without additional training samples by augmenting the input image into various views and training the learnable prompts to generate consistent responses across these different views. 

\subsection{Learning with Noisy Labels}
Mislabeled data can lead to deep neural networks overfitting to noisy labels. To address the issue of learning from noisy labels, previous researchers have proposed various methods, including robust network architectures \cite{leeRobustInferenceGenerative2019, yaoDeepLearningNoisy2018}, robust regularization techniques \cite{hendrycksUsingPretrainingCan2019, menonCanGradientClipping2020, xiaRobustEarlylearningHindering2020}, robust loss functions \cite{fengCanCrossEntropy2021, ghoshRobustLossFunctions2017, lyuCurriculumLossRobust2019}, correction of loss via estimation matrices \cite{changActiveBiasTraining2017, xiaAreAnchorPoints2019, yaoDualReducingEstimation2020}, and sample selection and meta-learning approaches \cite{liDividemixLearningNoisy2020, patelAdaptiveSampleSelection2023, songSelfieRefurbishingUnclean2019}.

The study of prompt learning with noise labels is currently in its nascent stage. A pioneering work by Wu et al. \cite{wuWhyPromptTuning2023} demonstrated that prompt learning is more robust than other parameter-efficient fine-tuning methods, such as adapters. Subsequently, JoAPR \cite{guoJoAPRCleaningLens2024} uses a Gaussian mixture model with joint adaptive thresholds to differentiate between clean and noisy data. It corrects labels by combining the results of data augmentations with mixup loss, and then retrains the model using this refined data. However, this approach does not fully leverage the benefits of prompt learning. In contrast, our study shows that using a simple MAE loss in prompt learning can boost the capability for handling noisy data. Additionally, we incorporate prompt-based optimal transport to further purify the noise samples. 

\subsection{Feature Learning Theory}
To further understand how noisy label learning affects prompt learning, we leverage feature learning theory to analyze the learning process. Feature learning theory \cite{allen-zhuUnderstandingEnsembleKnowledge2022,caoBenignOverfittingTwolayer2022,huang2023understanding,wen2021toward,huang2023graph,zou2021understanding,huang2024comparison,jiang2024unveil} categorizes latent features into task-relevant and task-irrelevant components, expressing the trainable weights as a combination of these feature types. From this perspective, feature learning analyzes the coefficients of these features to gain insight into learning dynamics. Beyond its application in traditional learning paradigms, prompt learning can also be explained by feature learning theory \cite{panFederatedLearningVisionLanguage2024}. In this paper, we adopt feature learning theory to demonstrate the robustness of MAE in prompt learning.

\section{Preliminary}

\textbf{Notation.} In our work, vectors are represented by lowercase bold letters, matrices by uppercase bold letters, and scalars by regular, non-bold letters. The \(\ell_2\)-norm of a vector \(\mathbf{v}\) is denoted as \(|\mathbf{v}|_2\). For matrices, the spectral norm of \(\mathbf{A}\) is indicated by \(|\mathbf{A}|_2\), and the Frobenius norm by \(|\mathbf{A}|_F\). The indicator function is represented by \(\mathds{1}(\cdot)\). Finally, sequences of integers are represented as \([n] = \{1, 2, \ldots, n\}\), and sequences of elements, such as vectors, are similarly denoted as \(\mathbf{v}_{[n]} = \{\mathbf{v}_1, \mathbf{v}_2, \ldots, \mathbf{v}_n\}\).\\

\noindent \textbf{Prompt Learning.} In this section, we demonstrate how to fine-tune a learnable text prompt within a vision-language pre-trained model. We focus on a classification task, where we have an image \(\mathbf{x}\) that we aim to classify into the correct ground truth class \(y \in [C]\), with \(C\) representing the total number of classes. From the vision-language pre-trained model, we expect the latent spaces of the text encoder and image encoder to be aligned. This alignment ensures that, when different prompts are input, the text feature generated by the correct prompt will have the highest similarity with the image feature. In this setup, we input a learnable prompt \(\mathbf{p} \in \mathbb{R}^d\) along with a fixed class prompt \(\mathbf{p}_c \in \{\mathbf{p}_1, \ldots, \mathbf{p}_C\}\), where each \(\mathbf{p}_c \in \mathbb{R}^d\) represents a specific class, into the text encoder \(h\). Here, \(d\) denotes the dimensionality of the prompts. By incorporating the learnable prompt $\mathbf{p}$, we generate the text feature for class \(c\) as \(\mathbf{h}_{c} = h(\mathbf{p}, \mathbf{p}_c) \in \mathbb{R}^m\). Meanwhile, the image feature \(\mathbf{g}\) is produced by the image encoder \(g\) as \(\mathbf{g} = g(\mathbf{x}) \in \mathbb{R}^m\). We define the similarity function between the image feature \(\mathbf{g}\) and the text feature \(\mathbf{h}_{c}\) as \(\bm{\rho}  = \text{sim}(\mathbf{g}, \mathbf{h}_{c})\in \mathbb{R}^c\). The training process follows the structure of traditional classification tasks, with an objective loss \(\ell(\bm{\rho}, \mathbf{e}_y)\) that measures the distance between the similarity vector \(\bm{\rho}\) and the true label \(y\). Here, \(\ell\) represents the loss function quantifying the distance between two vectors, and \(\mathbf{e}_y\) is the one-hot vector associated with the ground truth label \(y\).\\

\noindent \textbf{Optimal Transport.} Optimal transport (OT) is a constrained optimization problem that seeks to determine the optimal coupling matrix that maps one probability distribution to another while minimizing the total cost. Given the marginal distributions $\bm\alpha \in \mathbb{R}^n$, $\bm\beta \in \mathbb{R}^m$, and the cost matrix $\mathbf{C} \in \mathbb{R}^{n \times m}$, the classical OT problem is formulated as follows:
\begin{equation}
    \begin{aligned}
        \min\limits_{\mathbf{Q} \in \mathbb{R}_+^{n \times m}} \qquad &\langle \mathbf{C}, \mathbf{Q}\rangle\\
    \text{s.t.} \qquad & \mathbf{Q}\mathds{1}_{m} = \bm\alpha, \ \mathbf{Q}^\top\mathds{1}_n = \bm\beta.
    \end{aligned}
\end{equation}
This problem is a linear programming task, which becomes computationally expensive as the problem scale increases. To address this, Sinkhorn \cite{cuturiSinkhornDistancesLightspeed2013} proposed adding an entropic regularization, which allows for a closed-form solution and provides a "lightspeed" algorithm that only requires iterative scaling of the transportation matrix. The entropic regularized formulation is given as follows:
\begin{align}
    \min\limits_{\mathbf{Q} \in \mathbb{R}_+^{n \times m}} \qquad &\langle \mathbf{C}, \mathbf{Q}\rangle - \epsilon H(\mathbf{Q})\label{formu: EOT}\\
    \nonumber\text{s.t.} \qquad & \mathbf{Q}\mathds{1}_{m} = \bm\alpha, \ \mathbf{Q}^\top\mathds{1}_n = \bm\beta,
\end{align} 
where $H(\mathbf{Q}) = \sum_{i,j} Q_{ij} \left( \log Q_{ij} - 1 \right)$ and $\epsilon \geq 0$ is the coefficient that controls the regularization term.
In previous work, OT has been formulated as a pseudo-labeling technique for a range of machine learning tasks, including class-imbalanced learning \cite{guoLearningReweightExamples2022, wangSolarSinkhornLabel2022}, semi-supervised learning \cite{nguyenConfidentSinkhornAllocation2022, laiSarSelfadaptiveRefinement2022}, clustering \cite{asanoSelflabellingSimultaneousClustering2019, caronUnsupervisedLearningVisual2020, finiUnifiedObjectiveNovel2021}, domain adaptation \cite{zhengGroupawareLabelTransfer2021, changUnifiedOptimalTransport2022}, label refinery \cite{xiaOtCleanerLabel2022, wangSolarSinkhornLabel2022, fengOtfilterOptimalTransport2023,chang2023csot}, and others. Unlike prediction-based pseudo-labeling \cite{sohnFixmatchSimplifyingSemisupervised2020}, OT-based pseudo-labeling optimizes the mapping samples to class centroids, while considering the global structure of the sample distribution in terms of marginal constraints instead of per-sample predictions.

\section{Theoretical Analysis for the Robustness of PromptMAE}

In prompt learning with noisy labels, the ground truth class $y$ is flipped to a different noisy label class with a certain probability. As shown in Figure \ref{fig:cevsmae}, we compared the performance of the original CoOp using cross-entropy (CE) loss and mean absolute error (MAE) loss.  We observed that as the noise level increased in the datasets, the performance using CE loss significantly dropped, while the mean absolute error loss showed negligible change.
\vspace{2pt}
\\

\noindent \textbf{Basic Settings.} To explain this phenomenon, we apply feature learning theory to characterize the mechanism behind the robustness of MAE in the context of prompt learning. In our analysis, the objective is to classify the image $\mathbf{x}$ into its true class label $y$. In this theoretical analysis, we focus on a binary classification scenario where the class label $y_i \in \{+1, -1\}$. We also assume that the latent spaces of both the text encoder and the image encoder in the vision-language pre-trained model are well aligned. The latent feature space is assumed to consist of both task-relevant features, denoted as $\bm{\mu} \in \mathbb{R}^m$, and task-irrelevant features $\bm{\xi}_1, \dots, \bm{\xi}_L \in \mathbb{R}^m$, where $L$ represents the number of task-irrelevant features and the dimension of the latent space is $m$. For simplicity, we assume these features are orthogonal to each other. 
\\

\noindent\textbf{Text Encoder.}\quad Adopting a setup similar to \cite{wenUnderstandingFeatureLearning2021b}, a learnable prompt $\mathbf{p}$ and a fixed class prompt $\mathbf{p}_c$ are fed into the text encoder $h$:
\begin{equation}
\label{formu:def_of_text_encoder}
\begin{aligned}
    \mathbf{h}_{c} &= h(\mathbf{p}, \mathbf{p}_c)\\& = \sigma(\mathbf{W}\mathbf{p} + \mathbf{W}\mathbf{p}_c) - \sigma(-\mathbf{W}\mathbf{p} + \mathbf{W}\mathbf{p}_c),
\end{aligned}
\end{equation}
where $\mathbf{W} \in \mathbb{R}^{m \times d}$ is the weight matrix, and $\mathbf{p}_c \in \mathbb{R}^{d}$ is the prompt associated with class $c$. To examine the properties of the text encoder in the pre-trained model, we follow  \cite{huangWhatMakesMultimodal2021a,panFederatedLearningVisionLanguage2024} and set the weight matrix \(\mathbf{W}\) to be: 
\begin{align}
\label{formu: row of weight}
    \mathbf{W} &= \begin{bmatrix}
        \bm{\mu}, \bm{\xi}_1, \cdots, \bm{\xi}_L
    \end{bmatrix}^\top.
\end{align} 
\noindent\textbf{Image Encoder.} Let us consider the image network, represented as \(\mathbf{g}_i = g(\mathbf{x}_i) \in \mathbb{R}^m\). Due to that we assume the image encoder $g$ aligns the feature space of the text encoder $h$. As a result, the image feature generated by data \(\mathbf{x}_i\) in client \(i\) can be expressed as:
\begin{equation}
\label{formu: image_encoder}
    \mathbf{g}_i = g(\mathbf{x}_i) = [y_i, x_{i, 1}, \cdots, x_{i, L}]^\top,
\end{equation}
where \(x_{i, l} \sim \mathcal{N}(0, \sigma_p^2), \forall l \in [L]\) represents the coefficient of task-irrelevant terms in the data, and \(\sigma_p^2\) is the variance. The similarity score between an image \(\mathbf{x}_i\) and class \(y_i\) is given by \(\text{sim}(\mathbf{g}_i, \mathbf{h}_i) = \langle \mathbf{g}_i, \mathbf{h}_i \rangle\). To compute the probability, we first pass the logits of the similarity vector through the softmax function: 
\begin{align}
    \label{formu: similarity}
    s_i (\mathbf{p}) = \text{SOFTMAX}(\text{sim}(\mathbf{g}_i, \mathbf{h})).
\end{align}
The CE loss and MAE loss are defined as:
\begin{align}
    \ell_\text{CE}(\mathbf{s}_i, \mathbf{y}_i) = \sum\limits_{c=1}^C -y_{i, c} \log s_{i, c}, \\
    \ell_\text{MAE}(\mathbf{s}_i, \mathbf{y}_i) = \sum\limits_{c=1}^C | y_{i, c} - s_{i, c}|.
\end{align}

\begin{figure*}[htbp]
\centering
\includegraphics[width=\linewidth]{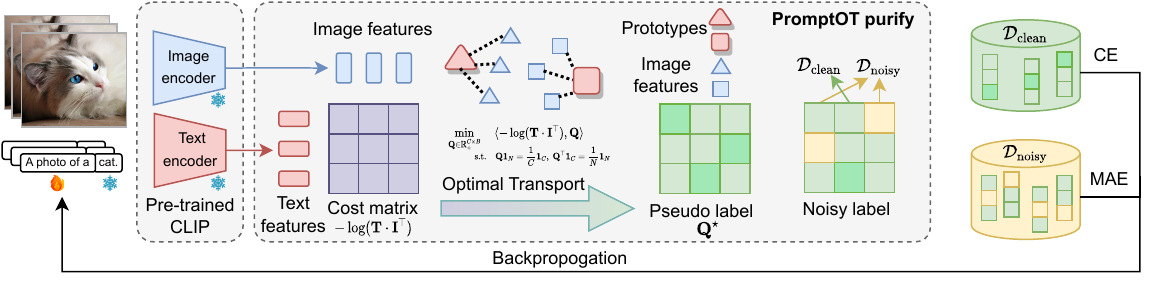}
 \caption{The framework of our NLPrompt. We utilize the text representation to initialize prompt-based OT, which separates the dataset into clean and noisy subsets. NLPrompt harmonizes the advantage of MAE loss and CE loss. The former is more robust on the noisy dataset while the latter performs better on the clean dataset.}
 \label{fig:alpha_user}
\end{figure*}

\noindent\textbf{Noisy Label Modeling.} Here, we introduce a model for label noise. We assume that the label noise follows a Rademacher random variable. Specifically, the noisy label $\tilde{y}$ is generated from the ground truth label $y$ with probability $p \leq 1/2$. That is, $\mathds{P}[\tilde{y} = -y] = p$ and $\mathds{P}[\tilde{y} = y] = 1 - p$. For the purpose of analysis, we divide the entire dataset into two subsets: the clean dataset $S_+ = \{ i \mid \tilde{y}_i = y_i \}$ and the noisy dataset $S_- = \{ i \mid \tilde{y}_i = -y_i \}$.
\\

\noindent\textbf{Feature Representation.} Under feature learning theory, the weight of the prompt can be decomposed into a combination of task-relevant features and task-irrelevant features. We present the following feature representation lemma \cite{panFederatedLearningVisionLanguage2024}:  

\begin{lemma} \label{main_lemma} 
    At the \(t\)-th iteration, the learnable prompt \(\mathbf{p}^{(t)}\) can be rewritten as a linear combination of the features and the prompt initialization:
    \begin{align}
        \nonumber \mathbf{p}^{(t)} &= \alpha^{(t)} \mathbf{p}^{(0)} + \beta^{(t)} ||\bm{\mu}||_2^{-2} \bm{\mu} + \sum\limits_{l = 1}^L \phi_{l}^{(t)} ||\bm{\xi}_{l}||_2^{-2} \bm{\xi}_{l},
    \end{align}
    where \(\alpha^{(t)}\) are the coefficients of the initialization, \(\beta^{(t)}\) and \(\phi_{l}^{(t)}\) are the coefficient of the task-relevant features and task-irrelevant features, respectively.
\end{lemma}

\noindent Since the learnable prompt can be expressed as a linear combination of the features, we can analyze the dynamics of these coefficients to understand the learning progress of the prompts. The normalization factor, such as \(||\boldsymbol{\mu}||_2^{-2}\), ensures that the coefficients are comparable to the inner product of the prompt and the features, i.e., \(\beta^{(t)} \approx \langle \mathbf{p}^{(t)}, \bm{\mu} \rangle\). \\

\noindent\textbf{Robustness of PromptMAE.} Building on the previous setup, we now examine how label flipping noise affects the learning dynamics of the coefficients. We will show that, the PromptMAE loss can enhance task-relevant coefficients for clean samples and help mitigate the degradation of task-relevant coefficients for noisy samples. 

By analyzing the dynamics of these coefficients, we gain deeper insights into the learning process. As shown in \cite{panFederatedLearningVisionLanguage2024}, the performance of prompt fine-tuning can be evaluated based on the ratio between task-relevant and task-irrelevant coefficients. To illustrate that MAE produces more robust results, we present the following theorem:

\begin{theorem}\label{T4.2}
    With high probability at test $1-d^{-1}$, the test loss $\ell_{\mathcal{D}}$ for the prompt trained by MAE is lower than the prompt trained by CE, i.e., $\ell_{\mathcal{D}}(\mathbf{p}_\text{MAE}) \leq \ell_{\mathcal{D}} (\mathbf{p}_\text{CE})$. 
    Here, $\mathbf{p}_\text{MAE}$ and $\mathbf{p}_\text{CE}$ refer to the text prompt trained using MAE loss and CE loss, respectively. 
\end{theorem} 

\noindent The proof is provided in the Appendix. From this result, we observe that under a noisy dataset, the MAE loss demonstrates greater robustness. 

\section{Methodology}

In this section, we introduce our NLPrompt algorithm and explain how we utilize the OT problem for data purification. NLPrompt harmonizes the advantage of MAE loss and CE loss. Theorem \ref{T4.2} has shown that MAE loss is more robust on the noisy dataset. 
\\

\noindent\textbf{PromptOT Purification.} 
Here, we utilize OT to generate pseudo-labels for data purification. Traditionally, the OT-based pseudo-labeling method starts with the random initialization of prototypes, and pseudo-labels are then derived from the similarity between images and these prototypes. However, in the context of prompt learning with vision-language foundation models, where the latent space is aligned, the randomly initialized prototypes can be replaced with text features generated by prompts via text encoders. The semantic information embedded in these text features provides a strong foundation for initialization.

Specifically, as outlined in Equation (\ref{formu: EOT}), the OT problem involves solving for a transportation matrix based on a given cost matrix, while preserving the marginal distributions. The similarity between the prototypes and the image features is calculated, and the negative logarithm of the resulting similarity matrix is used as the cost matrix. Due to the marginal distribution constraints, the columns of the OT matrix are normalized, and this matrix is then used as pseudo-labels for the images.

The calculation process in NLPrompt is outlined below. For images in dataset $\{\mathbf{x}_i\}_{i=1}^N$, we first use the pre-trained image encoder of CLIP to generate an image feature matrix $\mathbf{I} \in \mathbb{R}^{N \times d}$, where $d$ represents the dimension of the latent space. Additionally, given the set of classes, we generate prompts corresponding to these classes and pass them to the pre-trained text encoder of CLIP to create a text feature matrix $\mathbf{T} \in \mathbb{R}^{C \times d}$, where $C$ is the number of classes. Next, we calculate the similarity matrix  $\mathbf{T}\cdot\mathbf{I}^\top$. The negative logarithm of this similarity matrix is then used as the cost matrix in the OT problem, with uniform marginal distributions across both the samples and the classes. The OT problem to be solved is given as follows:
\begin{align}
    \label{prob: OT}
    \min\limits_{\mathbf{Q} \in \mathbb{R}_+^{C \times N}} \quad &\langle -\log(\mathbf{T}\cdot\mathbf{I}^\top), \mathbf{Q}\rangle \\
    \nonumber\text{s.t.} \quad & \mathbf{Q}\mathds{1}_{N} = \frac{1}{C}\mathds{1}_C, \ \mathbf{Q}^\top\mathds{1}_C = \frac{1}{N} \mathds{1}_N.
\end{align}

\noindent From this formulation, we obtain the OT matrix $\mathbf{Q}^\star$. We then apply the Argmax operation to each column of $\mathbf{Q}^\star$ to find the maximum value, i.e.,
\begin{equation*}
    \hat{y}_i = \arg \max_j \mathbf{Q}_{ij}.
\end{equation*}

\noindent\textbf{Harmonizing MAE and CE within NLPrompt.} The pseudo-labels generated by PromptOT are used to purify the dataset into two subsets: the clean dataset $\mathcal{D}_\text{clean}$ and the noisy dataset $\mathcal{D}_\text{noisy}$, defined as follows: 
\begin{align}
    \mathcal{D}_\text{clean} = \left\{i \mid \hat{y}_i = \tilde{y}_i \right\}, \quad \mathcal{D}_\text{noisy} = \left\{j \mid \hat{y}_j \neq \tilde{y}_j \right\}.
\end{align}
After the split, the two subsets are trained using different loss functions. For the clean dataset, we leverage the high performance of CE loss, while for the noisy dataset, we use MAE loss to enhance robustness. 
The harmonizing loss can be expressed as 
\begin{align}
\ell_\text{NLPrompt} &= \sum_{i \in \mathcal{D}_\text{clean}} -\mathbf{y}_{i}^\top \log \mathbf{s}_{i}  +  \sum_{j \in \mathcal{D}_\text{noisy}} ||\mathbf{y}_{j} - \mathbf{s}_j||_1.
\end{align}
where $\mathbf{y}_i$ is the target label and $\mathbf{s}_i$ is the output similarity for $i$-th sample.
\\

\noindent\textbf{Remark.} 
Our NLPrompt utilizes OT to harmonize CE and MAE. Unlike using generalized CE loss in noise-label prompt learning \cite{wuWhyPromptTuning2023}, our method fully exploits the advantages of prompt learning under vision-language foundation models. First, we utilize the text representation from prompt learning as a strong initial prototype. This allows our method to maintain global label consistency, setting it apart from other prediction-based methods. Additionally, we refine the dataset to take advantage of the robustness of mean absolute error, specifically for noisy samples, rather than treating both clean and noisy samples with the same loss. This flexibility not only enhances our model’s robustness but also allows us to leverage the advantages of CE, leading to improved overall performance.

\begin{table*}[htbp]
\centering
\caption{Performance metrics across various datasets and noise levels. (\%)\label{tab: main}}
\resizebox{\textwidth}{!}{%

\begin{tabular}{c|c|cccccc|cccccc}
\toprule
\multirow{2}{*}{\textbf{Dataset}} & \multirow{2}{*}{\textbf{Method}} & \multicolumn{6}{c|}{\textbf{Noise Rate: Sym}} & \multicolumn{6}{c}{\textbf{Noise Rate: Asym}} \\ [1.5pt]
                         &  & \textbf{12.5\%} & \textbf{25.0\%} & \textbf{37.5\%} & \textbf{50.0\%} & \textbf{62.5\%} & \textbf{75.0\%} & \textbf{12.5\%} & \textbf{25.0\%} & \textbf{37.5\%} & \textbf{50.0\%} & \textbf{62.5\%} & \textbf{75.0\%} \\ \midrule

\multirow{4}{*}{Flowers102} & CoOp & 88.93 & 83.50 & 77.93 & 70.10 & 55.60 & 37.17 & 86.97 & 74.70 & 60.43 & 42.60 & 26.53 & 12.60 \\
& GCE & 88.80 & 88.33 & 86.73 & 84.07 & 78.37 & 70.37 & 88.40 & 86.37 & 80.33 & 69.93 & 61.50 & 39.23 \\
& JoAPR & 85.57 & 81.23 & 74.60 & 70.23 & 67.90 & 66.93 & 85.17 & 79.63 & 73.97 & 73.83 & 53.37 & 13.27 \\
& NLPrompt & \textbf{93.87} & \textbf{92.57} & \textbf{92.73} & \textbf{89.90} & \textbf{84.77} & \textbf{76.80} & \textbf{93.80} & \textbf{93.40} & \textbf{91.77} & \textbf{81.10} & \textbf{73.63} & \textbf{55.33} \\ \midrule

\multirow{4}{*}{DTD} & CoOp & 56.00 & 49.57 & 43.30 & 34.37 & 27.83 & 17.27 & 55.60 & 47.75 & 38.07 & 29.63 & 20.53 & 11.70 \\
& GCE & 61.00 & 59.83 & 56.80 & 50.73 & 43.60 & 33.67 & 60.70 & 57.57 & 52.70 & 43.97 & 33.40 & 18.23\\
& JoAPR & 58.07 & 57.70 & 56.33 & 53.03 & 48.05 & 29.90 & 52.40 & 56.63 & 53.10 & 48.93 & 40.20 & 28.26  \\
& NLPrompt & \textbf{62.97} & \textbf{61.23} & \textbf{59.17} & \textbf{55.17} & \textbf{49.03} & \textbf{39.80} & \textbf{62.30} & \textbf{60.60} & \textbf{56.47} & \textbf{50.80} & \textbf{40.27} & \textbf{28.37} \\ \midrule

\multirow{4}{*}{EuroSAT} & CoOp & 76.50 & 69.23 & 61.67 & 52.33 & 37.63 & 26.70 & 76.00 & 66.27 & 53.83 & 41.17 & 28.00 & 17.43 \\
& GCE & 82.13 & 78.60 & 74.67 & 63.13 & 49.67 & 31.40 & 78.23 & 72.70 & 63.63 & 45.30 & 22.90 & 12.10\\
& JoAPR & 75.13 & 61.10 & 60.90 & 63.63 & 38.97 & 27.33 & 69.37 & 67.30 & 59.40 & 47.60 & 33.93 & 17.50 \\
& NLPrompt & \textbf{82.53} & \textbf{79.53} & \textbf{78.13} & \textbf{66.70} & \textbf{63.53} & \textbf{43.80} & \textbf{80.13} & \textbf{77.13} & \textbf{71.43} & \textbf{54.30} & \textbf{37.27} & \textbf{32.73} \\ \midrule

\multirow{4}{*}{OxfordPets} & CoOp & 76.50 & 66.73 & 60.33 & 47.03 & 35.77 & 24.60 & 76.10 & 66.20 & 52.53 & 38.73 & 26.63 & 14.90  \\
& GCE & 85.63 & 84.60 & 83.67 & 79.23 & 71.40 & 53.17 & 85.50 & 83.03 & 76.73 & 68.07 & 50.70 & 31.97 \\
& JoAPR & 84.00 & 83.26 & 83.20 & 83.10 & 82.40 & \textbf{74.40} & 82.90 & 83.40 & 79.07 & 75.84 & 52.74 & 43.57 \\
& NLPrompt & \textbf{86.17} & \textbf{86.00} & \textbf{85.33} & \textbf{84.87} & \textbf{83.63} & 70.77 & \textbf{86.00} & \textbf{84.97} & \textbf{82.40} & \textbf{77.53} & \textbf{66.33} & \textbf{48.60} \\ \midrule

\multirow{4}{*}{StanfordCars} & CoOp & 66.20 & 59.70 & 53.40 & 45.90 & 35.67 & 22.90 & 65.77 & 57.13 & 46.23 & 33.73 & 22.37 & 12.80 \\
& GCE & \textbf{69.70} & 66.40 & 66.47 & 63.77 & 59.25 & 50.87 & \textbf{70.00} & 66.45 & 61.23 & 53.67 & 39.65 & 26.60 \\
& JoAPR & 68.60 & 66.30 & 62.83 & 56.67 & 48.50 & 39.40 & 66.47 & 61.70 & 51.50 & 42.03 & 30.80 & 22.97 \\
& NLPrompt & 69.37 & \textbf{68.80} & \textbf{67.20} & \textbf{65.63} & \textbf{62.83} & \textbf{58.30} & 69.77 & \textbf{67.53} & \textbf{64.23} & \textbf{59.03} & \textbf{50.90} & \textbf{39.50} \\ \midrule

\multirow{4}{*}{UCF101} & CoOp & 69.03 & 63.40 & 58.23 & 49.73 & 40.83 & 26.30 & 67.23 & 58.07 & 46.47 & 34.43 & 23.67 & 13.17 \\
& GCE & 74.00 & \textbf{73.63} & 72.57 & 69.37 & 66.00 & 57.07 & 73.90 & 71.87 & 67.97 & 62.23 & 52.50 & 36.37 \\
& JoAPR & 72.83 & 71.17 & 70.37 & 67.63 & 65.30 & 57.67 & 72.07 & 69.80 & 64.10 & 59.17 & 56.07 & 47.46 \\
& NLPrompt & \textbf{74.83} & 73.40 & \textbf{72.83} & \textbf{70.33} & \textbf{68.10} & \textbf{60.53} & \textbf{74.90} & \textbf{73.53} & \textbf{71.03} & \textbf{65.97} & \textbf{58.97} & \textbf{49.27} \\ \bottomrule

\multirow{4}{*}{Caltech101} & CoOp & 86.43 & 81.03 & 76.73 & 70.90 & 61.33 & 46.90 & 84.93 & 75.23 & 62.87 & 49.43 & 33.57 & 20.33 \\
& GCE & \textbf{92.00} & 90.90 & \textbf{90.80} & 89.30 & 86.70 & 79.03 & 91.27 & \textbf{91.20} & 89.73 & 85.80 & 78.20 & 62.07 \\
& JoAPR  & 90.30 & 90.45 & 89.90 & 88.27 & 86.93 & 83.93 & 90.30 & 89.30 & 88.30 & 88.73 & 85.80 & 81.90 \\
& NLPrompt & 91.73 & \textbf{91.13} & 90.77 & \textbf{89.93} & \textbf{88.30} & \textbf{86.70} & \textbf{91.60} & 91.17 & \textbf{90.20} & \textbf{89.27} & \textbf{86.17} & \textbf{81.07} \\ \midrule

\end{tabular}%
}

\end{table*}

\section{Experiments}
In this section, we conduct comprehensive experiments to evaluate the performance of our method in noisy label scenarios, demonstrating the effectiveness of our method.
\subsection{Datasets and Baselines}
\textbf{Datasets.} To evaluate the performance of our method, we conduct experiments on seven synthetic noisy datasets: Caltech101~\cite{caltech101}, DTD~\cite{dtd}, EuroSAT~\cite{eurosat}, Flowers102~\cite{flowers102}, OxfordPets~\cite{oxfordpets}, StanfordCars~\cite{stanfordcars}, and UCF101~\cite{ucf101}. These representative visual classification datasets are used to simulate datasets with limited samples. They cover a variety of tasks, including general object classification, texture classification, fine-grained classification, action recognition, and satellite imagery recognition. Since these datasets do not contain noisy labels by default, we manually generate noisy labels for them. In addition, we conduct experiments on a real-world noisy label dataset, Food101N~\cite{food101n}, which inherently contains noisy labels and does not require manually synthesized noisy labels.
\\

\noindent \textbf{Baselines.} We compare our NLPrompt method with three baselines : CoOp \cite{CoOp}, CoOp+GCE \cite{wuWhyPromptTuning2023}, JoAPR \cite{guoJoAPRCleaningLens2024}. The latter two methods are specifically designed to tackle label noise in prompt learning for prompt learning in vision-language pretrained Models.

\subsection{Noise Setting}
For these synthetic noisy datasets, we introduce two types of noise: symmetric noise (denoted as Sym) and asymmetric noise (denoted as Asym). We only flip the labels of the training set in these datasets while keeping the test set unchanged. For symmetric noise, the clean labels in the training set are randomly flipped to other labels with equal probability. This means that labels within the same class can be incorrectly mapped to multiple different classes. For asymmetric noise, the clean labels in the training set are only flipped to a unique neighboring label, with labels within the same class being mapped exclusively to their successor class. Due to its stronger structural nature, asymmetric noise has a more significant negative impact on model performance, making it a stricter robustness test to evaluate the model's stability and adaptability when confronted with label noise. The goal of learning with noisy labels is to train a robust model on the noisy training set and achieve high accuracy on the clean test set.

\subsection{Implementation Details}
In our experiments, we adopt the same setup as CoOp \cite{CoOp} and JoAPR \cite{guoJoAPRCleaningLens2024} to ensure a fair comparison. We use the SGD optimizer with an initial learning rate of 0.002 and employ cosine annealing. Our model backbone is consistent with CoOp \cite{CoOp}, based on the pre-trained CLIP model \cite{CLIP}, utilizing either ResNet-50 or ViT-B/16 as the image encoder, with ResNet-50 as the default if not explicitly specified. We use a 63M parameter text transformer as the text encoder. The default number of training epochs is set to 200. Additionally, we employ 16 shared context tokens across all categories, with the class token placed at the end. We sample a 16-shot training set from each dataset and evaluate the model on the original test set. The reported experimental results are the averages of test accuracy from three runs with different seeds, with the highest accuracy highlighted in bold. All experiments were conducted using PyTorch \cite{paszkePytorchImperativeStyle2019} on a cluster equipped with NVIDIA A40 GPU. It is noted that our noise setting differs from that of JoAPR. Our setting is more challenging and practical. The differences between the two noise settings and more implementation details are provided in the Supplementary Materials.

\begin{table}[tbp]
\centering
\vspace{-1em}
\caption{Test accuracy (\%) on Food101N. \label{tab: food101n}}
\resizebox{0.8\linewidth}{!}{
\begin{tabular}{c|c|c|c|c}
\toprule
Method & CoOp & GCE & JoAPR & NLPrompt\\
\midrule
Accuracy & 69.50 & 71.32 & 72.57 & 76.46 \\
\bottomrule
\end{tabular}%
}
\vspace{-1em}
\end{table}

\subsection{Performance Comparison}
For the synthetic noisy datasets, we introduce noise of varying intensities, ranging from 12.5\% to 75.0\%, with an interval of 12.5\%. The experimental results are shown in Table \ref{tab: main}. In the vast majority of cases, our NLPrompt achieves state-of-the-art performance, with only a few instances where it performs almost identically to the best results, showing negligible differences. Furthermore, in scenarios with high levels of noise, our method consistently outperforms the other methods, showing a significant performance improvement. This demonstrates the effectiveness and superiority of our method in handling noisy labels in prompt learning. The experimental results on a real-world noisy dataset Food101N are shown in Table \ref{tab: food101n}, where our NLPrompt outperforms all baseline methods, further highlighting the superiority of our approach.

\begin{figure*}[htbp]
\centering
\includegraphics[width=0.9\linewidth]{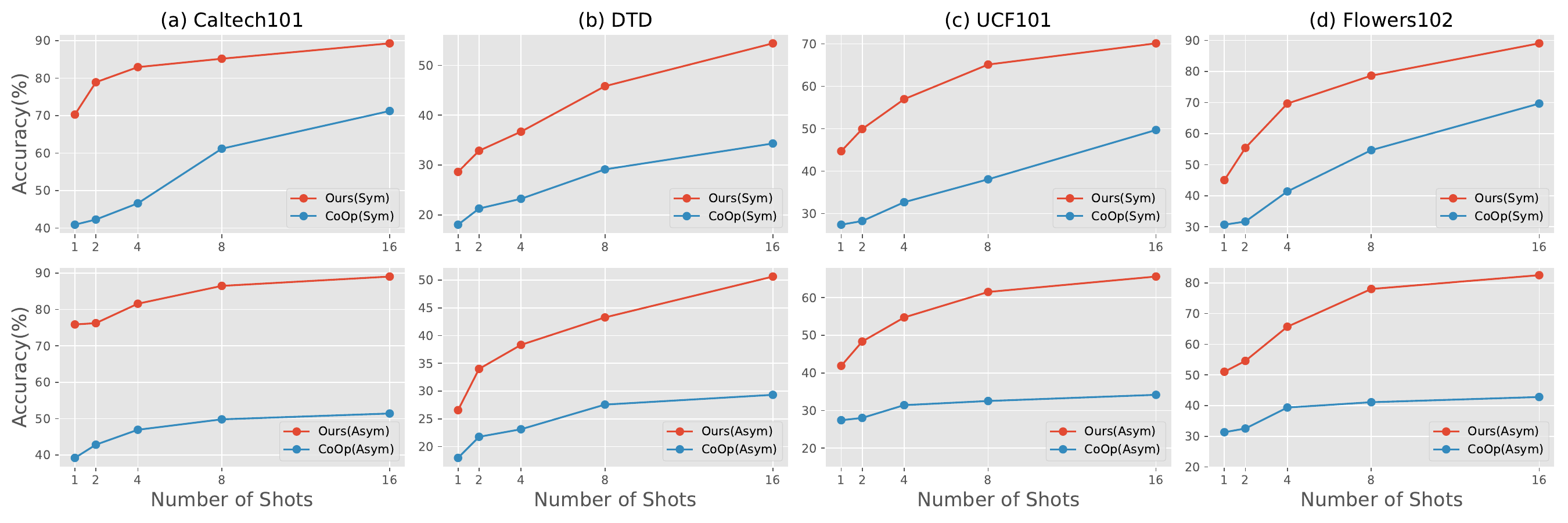}
 \caption{Performance with the different number of shots.}
 \label{fig:shot_num}
\end{figure*}

\subsection{The Generalization of NLPrompt}
Our method is effective not only for CoOp but also for other prompt-tuning approaches, including VPT \cite{jia2022visual}, Maple \cite{khattakMapleMultimodalPrompt2023}, and PromptSRC \cite{khattak2023self}, which are subsequent methods of CoOp. Additional results on the EuroSAT dataset under symmetric noise are shown in Table \ref{generalization}. It demonstrates the strong generalization ability of NLPrompt.
\begin{table}[tbp]
\vspace{-1em}
\setlength\tabcolsep{2pt}
\renewcommand\arraystretch{1}
\centering
\caption{The generalization of NLPrompt .\label{generalization}}
\label{result_promptsrc}
\resizebox{\columnwidth}{!}{
\begin{tabular}{c|c|c|c|c|c|c}
\toprule
Method/Noise Ratio    & 12.5\%  & 25.0\% & 37.5\% & 50.0\% & 62.5\% & 75.0\%   \\ \midrule
VPT & 89.20  & 79.43 & 65.20  & 61.37  & 41.67 & 27.67    \\
VPT+Ours  & \textbf{91.80}  & \textbf{91.07} & \textbf{89.53}  & \textbf{86.93}  & \textbf{80.73} & \textbf{73.90}    \\
\midrule
MaPLe & 83.27  & 77.67 & 65.27  & 55.40  & 37.53 & 25.47    \\
MaPLe+Ours  & \textbf{89.23}  & \textbf{84.30} & \textbf{78.37}  & \textbf{76.43}  & \textbf{73.30} & \textbf{59.87}    \\
\midrule
PromptSRC & 90.61  & 84.67 & 78.57  & 72.27  & 60.43 & 49.37    \\
PromptSRC+Ours  & \textbf{91.29}  & \textbf{87.67} & \textbf{84.97}  & \textbf{80.33}  & \textbf{72.10} & \textbf{59.50}    \\
\bottomrule
\end{tabular}}
\vspace{-2em}
\end{table}

\subsection{Few-shot Learning Analysis}

Following the few-shot evaluation setting used in previous work \cite{CoOp}, we further investigate the impact of the number of shots on different datasets. To this end, we vary the number of shots during training within the range of [1, 2, 4, 8, 16], while keeping the noise rate fixed at 50\%. The experimental results are shown in Figure \ref{fig:shot_num}, where the horizontal axis represents the number of shots and the vertical axis shows the test accuracy. We observe that as the number of shots increases, the performance of each method improves gradually. However, our method consistently outperforms the others, significantly enhancing the robustness of CoOp across different shot numbers and noise levels.

\begin{table}[htbp]
\centering
\caption{Ablation studies under multiple label noise ratios on Flowers102. (\%) \label{tab: ablation}}
\resizebox{\linewidth}{!}{%
\begin{tabular}{c|c|cccc|c}
\toprule
 & Method/Noise Ratio           & 10\%       & 30\%  & 50\%  & 70\%  & Avg \\ 
 \midrule
\multirow{2}{*}{w\textbackslash o OT} & (a) all data with CE          & 92.71 & 86.92 & 79.36 & 57.61 & 79.15\\ 
                            & (b) all data with MAE      & 88.47  & 89.07 & 85.20 & 80.87 & 85.90 \\ \midrule
\multirow{3}{*}{w\textbackslash \ OT}    & (c) w\textbackslash o text feature init & 87.16  & 83.81 & 79.77 & 73.00 & 80.94 \\ 
                            & (d) w\textbackslash o noisy data    & 84.77     & 84.53 & 81.60 & 77.60 & 82.16 \\ 
                            & (e) w\textbackslash o clean data   & 90.17  & 90.13 & 88.60 & 80.55 & 87.36 \\ \midrule
                            & NLPrompt               & \textbf{96.87}    & \textbf{93.44} & \textbf{92.30} & \textbf{85.38} & \textbf{92.00} \\ 
\bottomrule
\end{tabular}%
}
\end{table}
\subsection{Ablation Study}
To evaluate the effectiveness of each component of our method, we conduct ablation studies on the Flowers102 dataset. Here, we employ ViT-B/16 as the backbone and train for 100 epochs in the symmetric noise scenario. The experimental results are shown in Table \ref{tab: ablation}. 
To validate the effectiveness of OT, we designed two sets of experiments: one without using OT for data purification and another using OT for data purification. The experimental design is as follows: 
(a) Use CE loss for all data;
(b) Use MAE loss for all data;
(c) Use random initialization prototype instead of CLIP text feature as initialization;
(d) Use CE loss for clean data only after removing noisy data;
(e) Use MAE loss for noisy data only after removing clean data. 

The average results show that (b) outperforms (a), validating the effectiveness of our PromptMAE.
Moreover, the average results show that (d) outperforms (a), and (e) outperforms (b), further validating the effectiveness of PromptOT in the data purification process. Additionally, the comparison between (c) and NLPrompt highlights the importance of text feature initialization in our method. Among all methods, our NLPrompt achieves the best performance, with significant improvements over other baselines, further validating the effectiveness of each component.

\section{Conclusion}
In this study, we addressed the critical challenge of noisy labels in prompt learning for vision-language foundation models by introducing PromptMAE and PromptOT. Our findings demonstrate that adopting the MAE loss in prompt learning—despite its traditionally rare application in noisy-label scenarios—substantially enhances robustness and maintains high accuracy. By leveraging feature learning theory, we elucidated that MAE effectively suppresses the impact of noisy samples, thus improving the overall robustness. Furthermore, the introduction of PromptOT, a prompt-based OT data purification method, allows for an accurate partition of datasets into clean and noisy subsets. This selective application of CE loss to clean data and MAE loss to noisy data in NLPrompt underscores a simple yet powerful strategy for robust prompt learning. Extensive experiments conducted across various noise settings have confirmed the significant performance improvements. NLPrompt capitalizes on the expressive representation and precise alignment capabilities of vision-language models, presenting a promising solution to enhance the robustness of prompt learning in real-world scenarios. 
Extending NLPrompt to scenarios with unbalanced distributions is under consideration for the future work. 

\section*{Acknowledgement}
This work was supported by the National Natural Science Foundation of China (No.62406195, No.62303319), Shanghai Local College Capacity Building Program (23010503100), ShanghaiTech AI4S Initiative SHTAI4S202404, HPC Platform of ShanghaiTech University, MoE Key Laboratory of Intelligent Perception and Human-Machine Collaboration (ShanghaiTech University) and Shanghai Engineering Research Center of Intelligent Vision and Imaging. 

{
    \small
    \bibliographystyle{ieeenat_fullname}
    \bibliography{main}
}

\clearpage
\setcounter{page}{1}
\onecolumn
\begin{center}
{ \linespread{1.5} \selectfont
\textbf{\Large\thetitle} \\
}
\Large Supplementary Materials
\end{center}
\appendix

\renewcommand{\thetable}{A\arabic{table}}
\renewcommand{\thefigure}{A\arabic{figure}}
\makeatletter


\startcontents[appendix]
\printcontents[appendix]{l}{1}{\section*{Supplementary Organization:}\setcounter{tocdepth}{2}}
        

\section{Algorithm Framework}
NLPrompt employs optimal transport to enhance robust prompt learning by categorizing data into clean and noisy subsets and adapting different loss to each subset. The following pseudo-code illustrates the computation process for NLPrompt:
\begin{algorithm}[htbp]
\caption{NLPrompt: Optimal Transport-Based Data Partition for Robust Prompt Learning\label{main_algo}}

\begin{algorithmic}[1]
\State Initialize text encoder $h$, image encoder $\mathbf{g}$, class prompts $\mathbf{p}_c$ and trainable prompt $\mathbf{p}$

\For{each batch $\{\mathbf{x}_i\}_{i=1}^B$}
    \State Compute image features $\mathbf{I}$ and text features $\mathbf{T}$
    \State Compute similarity matrix $\mathbf{S} = \mathbf{T} \mathbf{I}^\top$
    \State Solve OT problem (\ref{prob: OT}) to get $\mathbf{Q}^\star$
    \State Generate pseudo-labels $\tilde{y}_i = \arg\max_j \mathbf{Q}_{ij}^\star$
    \State Partition data into $\mathcal{D}_\text{clean}$ and $\mathcal{D}_\text{noisy}$
    \For{each sample $(\mathbf{x}_{i}, \tilde{y}_{i})$}
        \If{$i \in \mathcal{D}_\text{clean}$}
            \State Use CE loss to update prompts
        \Else
            \State Use MAE loss to update prompts
        \EndIf
    \EndFor
\EndFor
\State \Return Fine-tuned text prompt $\mathbf{p}$ 
\end{algorithmic}
\end{algorithm}

\section{Details of Dataset Setup}
We selected eight representative visual classification datasets as benchmarks and manually added noise to create synthetic noisy datasets. Additionally, we included a real-world noisy dataset, Food101N, which inherently contains noise and does not require manual modification. Detailed statistics for each dataset, including the original task, the number of classes, and the sizes of training and test samples, are presented in Table \ref{detail}.
\begin{table}[htbp]
\centering
\caption{The detailed statistics of datasets used in experiments.\label{detail}}
\label{tab:datasets}
\begin{tabular}{@{}lccccc@{}}
\toprule
Noise Type & Dataset            & Task                                  & Classes & Training Size & Testing Size  \\ \midrule
\multirow{8}{*}{Synthetic noisy dataset} & Caltech101  & Object recognition                   & 100     & 4,128         & 2,465        \\
 & Flowers102  & Fine-grained flowers recognition      & 102     & 4,093         & 2,463         \\
 & OxfordPets  & Fine-grained pets recognition         & 37      & 2,944         & 3,669           \\
 & UCF101      & Video action recognition              & 101     & 7,639        & 3,783        \\
 & DTD         & Texture recognition                  & 47      & 2,820         & 1,692         \\ 
 & EuroSAT    & Satellite image classification        & 10      & 13,500        & 8,100          \\
 & StanfordCars     & Fine-grained car recognition    & 196     & 6,509        & 8,041         \\ 
 & SUN397   & Scene recognition  & 397     & 15,880        & 19,850        \\ \midrule
Real-world noisy dataset & Food101N  & Fine-grained food recognition & 101      &     310,009      &   30,300     \\ \bottomrule
\end{tabular}
\end{table}

\section{Further Experiments}
\subsection{Experiments on SUN397}
We also conducted experiments on SUN397, a dataset with a large number of classes. The results are presented in Table \ref{tab: further}. On this dataset, our NLPrompt still outperforms all baseline methods, further highlighting the superiority of our approach.
\begin{table*}[htbp]
\centering
\caption{Test accuracy (\%) on SUN397. \label{tab: further}}
\resizebox{\textwidth}{!}{%

\begin{tabular}{c|c|cccccc|cccccc}
\toprule
\multirow{2}{*}{\textbf{Dataset}} & \multirow{2}{*}{\textbf{Method}} & \multicolumn{6}{c|}{\textbf{Noise Rate: Sym}} & \multicolumn{6}{c}{\textbf{Noise Rate: Asym}} \\ [1.5pt]
                         &  & \textbf{12.5\%} & \textbf{25.0\%} & \textbf{37.5\%} & \textbf{50.0\%} & \textbf{62.5\%} & \textbf{75.0\%} & \textbf{12.5\%} & \textbf{25.0\%} & \textbf{37.5\%} & \textbf{50.0\%} & \textbf{62.5\%} & \textbf{75.0\%} \\ \midrule

\multirow{3}{*}{SUN397} & CoOp & 65.50 & 62.90 & 59.30 & 55.50 & 48.30 & 37.80 & 63.50 & 56.10 & 45.50 & 33.80 & 22.10 & 11.40 \\
& GCE & 67.60 & 66.30 & 65.40 & 64.20 & 62.00 & 59.20 & 68.40 & 66.40 & 63.80 & 60.00 & 53.60 & 43.80 \\
& NLPrompt & \textbf{68.40} & \textbf{67.50} & \textbf{66.40} & \textbf{64.80} & \textbf{64.10} & \textbf{61.70} & \textbf{68.70} & \textbf{67.50} & \textbf{66.10} & \textbf{64.00} & \textbf{61.40} & \textbf{53.00} \\ \midrule

\end{tabular}%
}

\end{table*}

\subsection{Experiments on Purfication Strategy}
We present experiments to demonstrate that PromptOT is an effective purification method in prompt learning. We frame the purification task as a binary classification problem, where the goal is to distinguish between clean and noisy samples. To evaluate the purification performance, we use accuracy and F1-score metrics for this binary classification task. We compare our purification strategy with the partition method that uses pseudo-labels generated by CLIP zero-shot (denoted as CLIP-ZS), as well as the partition strategy used in JoAPR \cite{guoJoAPRCleaningLens2024}. For our experiments, we use the Caltech101, DTD, and Flowers datasets, with results shown in Table \ref{tab: purification}. From this table, we observe that our prompt-based OT selection method achieves higher purification accuracy compared to both the CLIP zero-shot partition and the JoAPR partition. 
\begin{table}[htbp]
\centering

\caption{Comparison between different purification strategies. (\%) \label{tab: purification}}
\resizebox{0.45\linewidth}{!}{%

\begin{tabular}{c|cc|cc|cc}
\toprule
  \multirow{2}{*}{\textbf{Method}}        & \multicolumn{2}{c|}{\textbf{Caltech101}}  & \multicolumn{2}{c|}{\textbf{DTD}}  & \multicolumn{2}{c}{\textbf{Flowers}}  \\ 
   & Acc  & F1  & Acc  & F1 & Acc & F1 \\ 
 \midrule
CLIP-ZS              & 70.38 & 58.35 & 56.25& 26.72 & 56.62 & 25.16 \\ 
JoAPR& 50.88 & 44.99 & 51.54 & 32.32 & 50.37 &32.56\\
NLPrompt               & \textbf{75.37} & \textbf{67.41} & \textbf{59.97} & \textbf{36.89}& \textbf{62.07} & \textbf{39.49} \\ 
\bottomrule
\end{tabular}%
}
\end{table}

\subsection{Hyperparameter Ablation}
We evaluate the impact of different context token lengths on the DTD dataset under 50\% symmetric noise in Table \ref{tab: nctx} and test various entropy regularization coefficients of OT in Figure \ref{fig:regularization_entropy}. The results indicate that our NLPrompt is robust to hyperparameters.

\begin{table}[h]
\setlength\tabcolsep{2pt}
\renewcommand\arraystretch{1}
\centering
\caption{Test accuracy (\%) under different context token lengths. \label{tab: nctx}}
\resizebox{0.6\columnwidth}{!}{
\begin{tabular}{cccccc}
\hline
Context length            & 1       & 4 & 8 & 16 & 32  \\ \hline
NLPrompt & 53.65$\pm$1.95 &  54.00$\pm$1.49 & 54.07$\pm$0.40 & 55.20$\pm$0.33 & 54.73$\pm$1.52  \\\hline
\end{tabular}}
\end{table}

\begin{figure*}[h]
\centering
\includegraphics[width=0.75\linewidth]{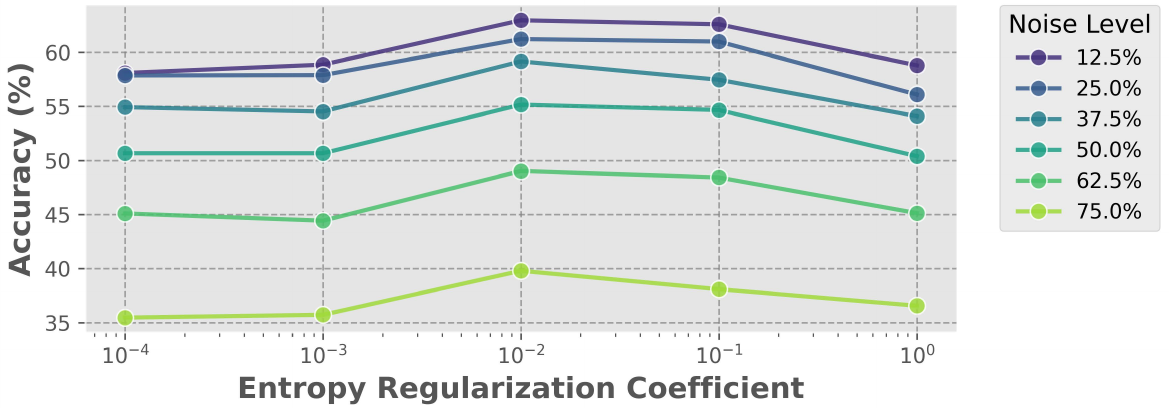}
 \caption{Test accuracy (\%) under different entropy regularization coefficients.} 
 \label{fig:regularization_entropy}
\end{figure*}

\subsection{Comparison with JoAPR}
The discrepancy in JoAPR results compared to the original paper arises from \textit{\textbf{different noise settings}}. JoAPR introduces noisy samples uniformly across each class, whereas we randomly distribute noisy samples throughout the entire training set. For example, at a 75\% noise rate with 16 shots, JoAPR's setting turns 12 out of 16 samples in each class into noisy samples and ensuring at least 4 clean samples per class. In contrast, our setting may lead to some classes having only 1 or even no clean samples. Thus, our setting is more challenging and practical, leading to significant performance variance at high noise levels. Additionally, we implemented our NLPrompt on the DTD dataset \textbf{\textit{using JoAPR’s settings}} and consistently observed improvements over JoAPR. The results are shown in Table \ref{tab:joapr}.
\begin{table}[htbp]
\setlength\tabcolsep{2pt}
\renewcommand\arraystretch{1}
\centering
\caption{Comparison with JoAPR in the JoAPR's noise setting.}\label{tab:joapr}
\resizebox{0.5\columnwidth}{!}{
\begin{tabular}{c|c|c|c|c|c|c}
\toprule
Method/Noise Ratio    & 12.5\%  & 25.0\% & 37.5\% & 50.0\% & 62.5\% & 75.0\%   \\ \midrule
JoAPR & 58.83  & 57.67 & 55.70  & 53.07  & 50.67 & 46.30   \\
NLPrompt & \textbf{63.17} & \textbf{61.96} & \textbf{60.82 } & \textbf{59.63}  & \textbf{53.83} & \textbf{49.60}   \\

\bottomrule
\end{tabular}}
\end{table}

\subsection{Comparison with Traditional LNL Method}
The traditional LNL method does not fully utilize the benefits of prompt learning with VLMs, whereas our method effectively leverage in prompt learning. In addition to the GCE method discussed in the paper, we also investigate the performance of Mixup on the EuroSAT dataset which is shown in Table \ref{tab:mixup}.

\begin{table}[htbp]
\setlength\tabcolsep{2pt}
\renewcommand\arraystretch{1}
\centering
\caption{Test accuracy (\%) compared with Mixup.}\label{tab:mixup}
\resizebox{0.5\columnwidth}{!}{
\begin{tabular}{c|c|c|c|c|c|c}
\toprule
Method/Noise Ratio    & 12.5\%  & 25.0\% & 37.5\% & 50.0\% & 62.5\% & 75.0\%   \\ \midrule
CoOp & 76.50  & 69.23 & 61.67  & 52.33  & 37.63 & 26.70   \\
CoOp+Mixup  & 75.30  & 71.63 & 64.07  & 54.23  & 42.37 & 26.90    \\
NLPrompt  & \textbf{82.53}  & \textbf{79.53} & \textbf{78.13}  & \textbf{66.70}  & \textbf{63.53}  & \textbf{43.80}  \\

\bottomrule
\end{tabular}}
\end{table}

\subsection{Computational Overhead of Optimal Transport}

We evaluate the computation time of optimal transport through experiments. Leveraging the efficient Sinkhorn algorithm for optimal transport, the computational overhead is minimal. For the Caltech101 dataset under a 16-shot learning setup, we perform batch matching, where 1600 image features are matched against 100 classes. The average time required for optimal transport per epoch is $0.00173$ seconds, while the original backward process takes $4.352$ seconds. Furthermore, even for significantly larger datasets, such as $100,000$ images $\times 1,000$ classes, the average computation time for the optimal transport method is approximately $1.888$ seconds, which remains relatively small compared to the overall training time.

\section{Limitation}
In this paper, we utilize prompt-based optimal transport for data purification, dividing the data into clean and noisy subsets. For the clean data, we apply cross-entropy (CE), while for the noisy data, we employ mean absolute error (MAE). As shown in Table \ref{tab: main}, our experimental results indicate that although NLPrompt achieves state-of-the-art performance in most cases, its performance is not always superior to other methods under low noise rates, with slight gaps compared to the best results. This suggests room for improvement in scenarios with low noise levels. This may be because, at low noise levels, optimal transport can misclassify some correct samples as noisy, leading to reduced performance of MAE on datasets with low noise rates.

\section{Theoretical Analysis for the Robustness of PromptMAE}
In this section, we demonstrate that the mean absolute error (MAE) loss is robust for prompt learning in vision-language foundation models. Leveraging the properties of vision-language pre-trained models, we assume that the latent spaces of the text encoder and image encoder are well-aligned. To clarify, we restate and explain some of our analysis settings. For a classification task, the objective is to classify an image $\mathbf{x}$ into its ground truth class $y \in [C]$, where $C$ represents the total number of classes. For simplicity, we assume that the features corresponding to these classes are orthogonal. In our theoretical analysis, we focus on a binary classification scenario, where \(y_i \in \{+1, -1\}\). In most theoretical work of feature learning, it is common to apply insights from binary classification to interpret experimental observations [1, 3]. So, we employ such theoretical frameworks to validate and support our experimental findings.

\noindent \textbf{Outline} In our proof, we begin by introducing the assumptions and feature modeling. We then analyze the gradient update using the chain rule and explore the relationship between the feature space and the gradient update. Next, we utilize the decomposition of trainable parameters to demonstrate how the decomposition coefficients change. By establishing the connection between the coefficients and the test loss, we compare the performance of different loss functions by examining the relationship between their respective coefficients.

 Following the standard feature learning theory \cite{allen-zhuUnderstandingEnsembleKnowledge2022}, we assume that the weights of the pretrained model consist of two components: task-relevant weights $\bm{\mu}$ and task-irrelevant weights $\bm{\xi}$. We begin by proving the following feature representation lemma.
\begin{lemma}[\textbf{Restatement of Lemma \ref{main_lemma}: Feature Representation}]
    \label{supp:main_lemma}
    At the \(t\)-th iteration, the learnable prompt \(\mathbf{p}^{(t)}\)  can be rewritten as a linear combination of features and prompt initialization:
    \begin{align}
        \nonumber \mathbf{p}^{(t)} &= \alpha^{(t)} \mathbf{p}^{(0)} + \beta^{(t)} ||\bm{\mu}||_2^{-2} \bm{\mu}  + \sum\limits_{l = 1}^L \phi_{l}^{(t)} ||\bm{\xi}_{l}||_2^{-2} \bm{\xi}_{l}, 
    \end{align}
    where \(\alpha^{(t)}\) are the coefficients of initialization, \(\beta^{(t)}\) is the coefficient of task-relevant features, \(\phi_{(\cdot)}^{(t)}\) are the coefficients of task-irrelevant features.
\end{lemma}
\noindent\textbf{Disuccusion on feature decomposition intuition}  Decomposite the trainable parameters in latent space into linear combinations is a general technique in feature learning theory [1, 3]. Intuitively, in a text prompt, certain "core" words (such as adjectives describing class features before class names) play a key role in determining the image classification and are considered "task-relevant," while other words are "task-irrelevant." 

\noindent\textbf{Coefficient dynamics} Inspired by the previous study \cite{caoBenignOverfittingTwolayer2022}, we analyze the dynamics of coefficients in prompt fine-tuning from vision-language foundation models. By analyzing the dynamics of the coefficients, we can reveal the feature learning procedure during training. This analysis allows us to establish the order of coefficients and explore how they are affected by the noisy rate \(p\). 

\noindent \textbf{Loss design} Our goal here is to compare two different types of loss functions: cross-entropy loss and mean absolute error loss.
\begin{align}
    \ell_\text{CE}(\mathbf{s}_i, \mathbf{y}_i) = \sum\limits_{c=1}^C -y_{i, c} \log s_{i, c} = -\mathbf{y}_i\log \mathbf{s}_i, \\
    \ell_\text{MAE}(\mathbf{s}_i, \mathbf{y}_i) = \sum\limits_{c=1}^C | y_{i, c} - s_{i, c}| = ||\mathbf{y_i - \mathbf{s_i}}||_1
\end{align}

\noindent Our analysis will be made under the following assumption: 
\begin{assumption}\label{assum:numerical}Suppose that:
    \begin{itemize}
        \item The number of training samples $N = \Omega(\text{polylog}(d))$, where $d$ is the dimension of learnable prompts.
        \item The dimension of latent space $m$ is sufficiently large, i.e., $m = \tilde{\Omega}(N)$.
        \item The learning rate $\eta \leq \tilde{O}(\min \{||\bm\mu||_2^2, \sigma_p^{-2}m^{-1}\})$ and the standard deviation of network weight initialization $\sigma_0 \leq \tilde{O}(mn)\cdot \min \{(||\bm\mu||_2^2, \sigma_p\sqrt{d}^{-1})\}$.
        \item we assume that $\bm{\mu}^T \mathbf{p}_{+1} \geq 0 \geq \bm{\mu}^T \mathbf{p}_{-1}$ which implies a separability condition in the latent space. 
    \end{itemize}
\end{assumption}

\noindent\textbf{Remark.} In this assumption, a sufficiently large number of training samples and latent dimension are used to ensure that the network has concentration properties. Meanwhile, a sufficiently small learning rate and appropriate weight initialization are employed to guarantee that gradient descent, in the theoretical analysis, leads to loss convergence.

\noindent\textbf{Gradient analysis.} In the definition of the text encoder (\ref{formu:def_of_text_encoder}), the incorporation of \(\mathbf{W}\mathbf{p}_c\) introduces nonlinearity between the trainable prompt and the class prompt while maintaining the overall function's nonlinear nature. The assumptions regarding the image encoder (\ref{formu: image_encoder}) suggest that task-relevant features differ depending on whether the label is positive or negative, while task-irrelevant features remain arbitrary and independent of the label's polarity. Furthermore, the training loss objective is designed to strengthen the similarity between the image feature \(g(\mathbf{x}_{i})\) and the text feature generated by the label class prompt \(\mathbf{p}_{y_{i}}\). This section discusses the computational approach used to analyze the performance of the algorithm, focusing on the gradient calculations essential for optimizing the model parameters. Leveraging the properties of the gradient of the Softmax function, we have: 
\begin{align}
    \frac{\partial s_{i, c}}{\partial \textbf{sim}(\mathbf{g_i}, h_c)} = (1-s_{i, c})s_{i, c}, \qquad
    \frac{\partial s_{i, c}}{\partial \textbf{sim}(\mathbf{g_i}, h_{-c})} = s_{i, -c} s_{i, c}.
\end{align}
We can derive the gradient of $s_{i,c}$ with respect to learnable prompt $\mathbf{p}$ with chain rule. 
\begin{align}
    \label{formula: gradient of logit}
    \frac{\partial s_{i, c}}{\partial \mathbf{p}} &= (1-s_{i, c})s_{i, c} \frac{\partial{\textbf{sim}(\mathbf{g}_i, \mathbf{h}_{c})}}{\partial \mathbf{p}}-s_{i, c}s_{i, -c}\frac{\partial{\textbf{sim}(\mathbf{g}_i, \mathbf{h}_{-c})}}{\partial \mathbf{p}}\\
    &= s_{i, c}s_{i, -c}\left(\frac{\partial{\textbf{sim}(\mathbf{g}_i, \mathbf{h}_{c})}}{\partial \mathbf{p}} -\frac{\partial{\textbf{sim}(\mathbf{g}_i, \mathbf{h}_{-c})}}{\partial \mathbf{p}}\right).
\end{align}
To calculate the gradient, we need the partial derivatives of $\mathbf{sim}(\mathbf{g}_{i}, \mathbf{h}_{c})$ with respect to $\mathbf{p}$. Recall from (\ref{formu:def_of_text_encoder}) and (\ref{formu: similarity}) that the similarity is defined as:  
\begin{align}  
    \textbf{sim}(\mathbf{g}_i, \mathbf{h}_c) = \langle \sigma(\mathbf{Wp} + \mathbf{Wp}_c) - \sigma(-\mathbf{Wp} + \mathbf{Wp}_c), g(\mathbf{x}_i) \rangle.  
\end{align}  
The gradient of this similarity can be derived as:
\begin{align}
\frac{\partial \mathbf{sim}(\mathbf{g}_i, \mathbf{h}_c)}{\partial \mathbf{p}} =(\mathbf{W}^T& (\sigma'(\mathbf{W}\mathbf{p} + \mathbf{W}\mathbf{p}_c)+\sigma'(-\mathbf{W}\mathbf{p} + \mathbf{W}\mathbf{p}_{c})))\cdot g(\mathbf{x}_i).
\end{align}
Using the previously derived gradient, we have
\begin{align}
    \frac{\partial \ell}{\partial \mathbf{p}} &= \frac{\partial \ell}{\partial s_{i, \tilde{y}_i}}\frac{\partial s_{i, \tilde{y}_i}}{\partial \mathbf{p}}= \frac{\partial \ell}{\partial s_{i, \tilde{y}_i}}s_{i, \tilde{y}_i}s_{i, -\tilde{y}_i}\left(\frac{\partial{\textbf{sim}(\mathbf{g}_i, \mathbf{h}_{\tilde{y}_i})}}{\partial \mathbf{p}} - \frac{\partial{\textbf{sim}(\mathbf{g}_i, \mathbf{h}_{-\tilde{y}_i})}}{\partial \mathbf{p}}\right).
\end{align}
The gradient for different loss functions can then be computed as follows:
\begin{align}
    \frac{\partial \ell_{CE}}{\partial s_{i, \tilde{y}_i}} &= -\frac{1}{s_{i, \tilde{y}_i}}, \qquad \frac{\partial \ell_{MAE}}{\partial s_{i, \tilde{y}_i}} = -2
\end{align}
Here, we define $\ell'_{i} = \frac{\partial \ell}{\partial s_{i, \tilde{y}_i}}s_{i, \tilde{y}_i}s_{i, -\tilde{y}_i}$ as the gradient coefficient. For cross-entropy loss, this simplifies to $\ell'_{i} = s_{i, -\tilde{y}_i}$, while for mean absolute error, it becomes $\ell'_{i} = 2 s_{i, \tilde{y}_i}s_{i, -\tilde{y}_i}$. Defining $\sigma'_{r, i} := \sigma(\mathbf{w}_r^T \mathbf{p} + \mathbf{w}_r^T \mathbf{p}_{\tilde{y}_i}) + \sigma(-\mathbf{w}_r^T \mathbf{p} + \mathbf{w}_r^T \mathbf{p}_{\tilde{y}_i}) - \sigma(\mathbf{w}_r^T \mathbf{p} + \mathbf{w}_r^T \mathbf{p}_{-\tilde{y}_i}) - \sigma(-\mathbf{w}_r^T \mathbf{p} + \mathbf{w}_r^T \mathbf{p}_{-\tilde{y}_i})$, the gradient can be expressed as follows:
\begin{align}
    \nabla_\mathbf{p}L_{\mathcal{T}}(\mathbf{p}) = -\frac{1}{n}\sum\limits_{i=1}^n \sum\limits_{r=1}^m \ell'_i x_{r, i}\sigma'_{r, i} \mathbf{w}_r.
\end{align}
Due to the update rule of gradient descent and the previous gradient formula, we can rewrite the update equation as follows:
\begin{align}
    \mathbf{p}^{(t+1)} &= \mathbf{p}^{(t)} - \eta \nabla_{\mathbf{p}} L_{\mathcal{T}}(\mathbf{p}^{(t)})\\
    &= \mathbf{p}^{(t)} + \frac{\eta}{n}\sum\limits_{i=1}^n \sum\limits_{r=1}^m \ell'_i x_{r, i}\sigma'_{r, i} \mathbf{w}_r,
\end{align}
where $\eta \geq 0$ is the learning rate. Next, we use the assumption about the rows of the weight matrix, as shown in (\ref{formu: row of weight}). This leads to the update formula for the corresponding rows of the features:
\begin{align}
    \label{formu: iteration of beta}\beta^{(t+1)} &= 
    \beta^{(t)} + \frac{\eta}{n}\sum\limits_{i=1}^n \ell'_i \sigma'_{1, i} \tilde{y}_iy_i ||\bm{\mu}||_2^2\\
    \nonumber&=\beta^{(t)} + \frac{\eta}{n}\sum\limits_{i\in S_+} \ell'_i \sigma'_{1, i} ||\bm{\mu}||_2^2  - \frac{\eta}{n}\sum\limits_{i\in S_-} \ell'_i \sigma'_{1, i} ||\bm{\mu}||_2^2\\
    \label{formu: iteration of phi}\phi_l^{(t+1)} &= \phi^{(t)} +  \frac{\eta}{n} \sum\limits_{i=1}^n \ell'_i \sigma'_{1, i} \tilde{y}_i  x_{l+1, i}||\bm{\xi}_l||_2^2.
\end{align}


\subsection{Theoretical analysis}
Our analysis follows this logic: both CE and MAE will increase task-relevant and task-irrelevant features for clean data. However, for noisy data, this leads to a decrease in task-relevant features and an increase in task-irrelevant features, causing the SNR to decrease for both. 
Inspired by \cite{panFederatedLearningVisionLanguage2024}, we introduce the following lemma. 
\begin{lemma}
    Under prompt learning, the test loss can be evaluated with the ratio between the task-relevant coefficient and task-irrelevant coefficient. 
\end{lemma}
\begin{proof}
    For simplicity, we first introduce the definitions of $F_{+}$ and $F_{-}$. Here $F_{+}$ means the train loss corresponding to the positive class, while $F_{-}$ means the train loss corresponding to the negative class.
\begin{align}
    &F_{+}(\mathbf{p}) = \sigma(\mathbf{W}\mathbf{p}+ \mathbf{W}\mathbf{p}_{+}) - \sigma(-\mathbf{W}\mathbf{p}+ \mathbf{W}\mathbf{p}_{+}),\\
    &F_{-}(\mathbf{p}) = \sigma(\mathbf{W}\mathbf{p}+ \mathbf{W}\mathbf{p}_{-}) - \sigma(-\mathbf{W}\mathbf{p}+ \mathbf{W}\mathbf{p}_{-}),\\
    &F(\mathbf{p}) = F_{+}(\mathbf{p}) - F_{-}(\mathbf{p}). \label{formu: introduce of F}
\end{align}
From (\ref{formu: introduce of F}), we have that the following expressions are equivalent:
\begin{align}
    \langle F_{y_i}(\mathbf{p}) - F_{-y_i}(\mathbf{p}),  g(\mathbf{x}_i)\rangle \geq 0 \Longleftrightarrow y_i\langle F(\mathbf{p}), g(\mathbf{x}_i)\rangle \geq 0.
\end{align}
Note that
\begin{align}
        \mathbf{W}\mathbf{p}= \begin{bmatrix}
            \bm\mu^\top \mathbf{p}\\
            \bm\xi_1^\top \mathbf{p}\\
            \vdots\\
            \bm\xi_L^\top \mathbf{p}
        \end{bmatrix}=\begin{bmatrix}
            \beta ||\bm\mu||_2 \\
            \phi_1 ||\bm\xi_1||_2\\
            \vdots\\
            \phi_L ||\bm\xi_L||_2
        \end{bmatrix}.
\end{align}
Also, since the weight and class prompts are fixed during prompt learning, $\mathbf{Wp}_{+}$ and $\mathbf{Wp}_{-}$ can be treated as two constant terms. To assess the algorithm's performance, we evaluate the error rate during the testing procedure, which serves as our test loss, denoted as $\ell_{\mathcal{D}}$: 
\begin{equation}
    \begin{aligned}
         \ell_{\mathcal{D}}(\mathbf{p}) &=\frac{1}{n}\sum\limits_{i=1}^{n} \mathds{1}(\hat{y}_{i} = y_{i}).
    \end{aligned}
\end{equation}
Recall that the test error is minimized when $s_{i, y_i}$ exceeds $s_{i, -y_i}$. Therefore, the accuracy of the $i$-th sample is equivalent to:
\begin{align}
    \mathds{1}(\hat{y}_{i} = y_{i}) = \mathds{P}(s_{i, y_i} - s_{i, -y_i} > 0).
\end{align}

\noindent According to the monotonicity of the softmax function, we have that 
\begin{align}
    s_{i, y_i} - s_{i, -y_i} > 0 \Longleftrightarrow \textbf{sim}(\mathbf{g}_i, h_{y_i}) - \textbf{sim}(\mathbf{g}_i, h_{-y_i})> 0.
\end{align}
Considering each row of $F(\mathbf{p}^{(t)})$ and $\mathbf{x}$, we can expand the expression as:
\begin{align}
\label{formu: equivalent expression}
    F_1(\mathbf{p}^{(t)}) + \sum\limits_{l=1}^L y_i x_{i, l}F_{l+1}(\mathbf{p}^{(t)}) \geq 0,
\end{align}
where $x_{i,r}$ are Gaussian random variables with zero mean, and $y_i$ are random variables independent of $x_{i, r}$. As in the test procedure, $F_1(\mathbf{p}), \dots, F_{l+1}(\mathbf{p})$ are fixed. Therefore, based on the definition of prompt learning, the expectation that each sample is correctly classified is determined by the ratio between $F_1(\mathbf{p})$ and the task-irrelevant coefficients $F_2(\mathbf{p}), \dots, F_{L+1}(\mathbf{p})$.

Since $F_1(\mathbf{p})$ is a monotonic function with respect to $\mu^{(t)}$ and $F_{l+1}(\mathbf{p})$ is a monotonic function with respect to $\phi_l^{(t)}$, we can express the following relationship:
\begin{align}
\ell_D(\mathbf{p}^{(t)}) \sim \frac{F_1(\mathbf{p}^{(t)})}{\sum\limits_{l=1}^L F_{l+1}(\mathbf{p}^{(t)})}.
\end{align}
Therefore, the test loss can be analyzed by evaluating the ratio between the task-relevant and task-irrelevant coefficients. Besides, due to the assumption that $\bm{\mu}^T \mathbf{p}_c \geq 0 \geq \bm{\mu}^T \mathbf{p}_{-c}$ and the random initialization of $\mathbf{Wp}$ such that $\bm\mu^T \mathbf{p} \geq 0$, we conclude that $F_1(\mathbf{p}) \geq 0$. Additionally, the expectation of $x_i$ is zero, and $F_{l}(\mathbf{p})$ remains constant for any iteration with a fixed $\mathbf{p}$, for all $l$. As a result, we derive the following inequality:  
\begin{align}  
    \mathds{E}(s_{i, y_i} - s_{i, -y_i}) \geq 0 \Longleftrightarrow \mathds{E}(s_{i, y_i} - (1 - s_{i, y_i})) \geq 0 \Longleftrightarrow \mathds{E}(s_{i, y_i}) \geq \frac{1}{2},  
\end{align}  
which can be used to analyze the feature dynamics in prompt learning. 
\end{proof}

Based on the previous lemma, we can derive the following theorem.
\begin{theorem}[\textbf{Restatement of Theorem \ref{T4.2}}]
   With high probability at test $1-d^{-1}$, the test loss $\ell_{\mathcal{D}}$ for the prompt trained by MAE is lower than the prompt trained by CE, i.e., $\ell_{\mathcal{D}}(\mathbf{p}_\text{MAE}) \leq \ell_{\mathcal{D}} (\mathbf{p}_\text{CE})$. 
    Here, $\mathbf{p}_\text{MAE}$ and $\mathbf{p}_\text{CE}$ refer to the text prompt trained using MAE loss and CE loss, respectively. 
\end{theorem}
\begin{proof}
    To prove that the MAE achieves better generalization performance, we need to compare the ratio between task-relevant and task-irrelevant coefficients. When the task-relevant features dominate, the algorithm performs better, whereas the dominance of task-irrelevant features leads to worse performance. Based on the iteration formulas in (\ref{formu: iteration of beta}) and (\ref{formu: iteration of phi}), we can derive the expected updates for $\beta$ and $\phi$:
    \begin{align}
        \beta^{(t+1)}_{\text{CE}} &= \beta^{(t)}_\text{CE} + \eta\left[(1-p) \frac{1}{\mathds{E}[s_{y}]} ||\bm\mu||_2^2 - p\frac{1}{1-\mathds{E}[s_y]} ||\bm\mu||_2^2\right], \\
        \phi^{(t+1)}_{\text{CE}} &= \phi^{(t)}_\text{CE} + \eta\left[(1-p) \frac{1}{\mathds{E}[s_{y}]} \sigma_p^2 d + p \frac{1}{1-\mathds{E}[s_y]} \sigma_p^2 d\right], \\
        \beta^{(t+1)}_{\text{MAE}} &= \beta^{(t)}_{\text{MAE}} + \eta\left[(1-p) \cdot 2 \cdot ||\bm\mu||_2^2 - p \cdot 2 \cdot ||\bm\mu||_2^2\right], \\
        \phi^{(t+1)}_{\text{MAE}} &= \phi^{(t)}_{\text{MAE}} + \eta\left[ (1-p) \cdot 2 \cdot \sigma_p^2 d + p \cdot 2 \cdot \sigma_p^2 d\right].
    \end{align}
    In this part, we use (\ref{formu: equivalent expression}) and the assumption that \( \bm\mu^\top \mathbf{p}_{+1} \geq 0 \geq \bm\mu^\top \mathbf{p}_{-1} \) to show that the expectation of \( F_1(\mathbf{p}^{(t)}) \) is greater than zero. Building on this, we assume that \( x_{l,i} \) are Gaussian random variables with zero mean, and \( y_i \) are Rademacher random variables independent of \( x_{l,i} \). Taking the expectation of the left-hand side of (\ref{formu: equivalent expression}), we obtain:
    \begin{align}
        \mathds{E}[F_{1}(\mathbf{p}^{(t)}) + \sum\limits_{l=1}^L y_i x_{i, l} F_{l+1}(\mathbf{p}^{(t)})] = \mathds{E}[F_1(\mathbf{p}^{(t)})] \geq 0,
    \end{align}
    which simplifies to:
    \begin{equation}
        \begin{aligned}
            \mathds{E}[s_{i, y_i} - s_{i, -y_i}] \geq 0 \ \Longleftrightarrow \ \mathds{E}[s_{i, y_i} - (1 - s_{i, y_i})] \geq 0 \ \Longleftrightarrow \ \mathds{E}[2s_{i, y_i} - 1] \geq 0 \ \Longleftrightarrow \ \mathds{E}[s_{i, y_i}] \geq \frac{1}{2}.
        \end{aligned}
    \end{equation}
    Here, the first inequality follows from the definition of accurate classification, the second inequality comes from the properties of the softmax function, and the fourth inequality arises from the properties of expectation. Based on this, the increment of the coefficients per step can be evaluated using expectation. We get the following expressions for the ratio of updates:
    \begin{align}
        \frac{\Delta\beta^{(t)}_\text{MAE}}{\Delta\beta^{(t)}_\text{CE}} &= \frac{(1-p)\frac{1}{\mathds{E}[s_y]} - p\frac{1}{1-\mathds{E}[s_y]}}{2 - 4p} = \frac{1}{2\mathds{E}[s_y]} \cdot \frac{1 - p \frac{1}{1 - \mathds{E}[s_y]}}{1 - 2p}, \\
        \frac{\Delta\phi^{(t)}_\text{MAE}}{\Delta\phi^{(t)}_\text{CE}} &= \frac{(1-p) \frac{1}{\mathds{E}[s_y]} + p \frac{1}{1-\mathds{E}[s_y]}}{2s} = \frac{1}{2\mathds{E}[s_y]} \left( 1 - p \frac{2\mathds{E}[s_y] - 1}{1 - \mathds{E}[s_y]} \right).
    \end{align}
    Given that \( \mathds{E}[s_y] > \frac{1}{2} \), we have \( \frac{1}{1 - \mathds{E}[s_y]} > 2 \) and \( 2 \mathds{E}[s_y] - 1 > 0 \). Additionally, since \( 0 \leq p \leq \frac{1}{2} \), it follows that:
    \begin{align}
        \frac{\Delta\beta^{(t)}_\text{MAE}}{\Delta\beta^{(t)}_\text{CE}} &= \frac{1}{2 \mathds{E}[s_y]} \cdot \frac{1 - p \frac{1}{1 - \mathds{E}[s_y]}}{1 - 2p} > \frac{1}{2 \mathds{E}[s_y]}, \\
        \frac{\Delta\phi^{(t)}_\text{MAE}}{\Delta\phi^{(t)}_\text{CE}} &= \frac{1}{2 \mathds{E}[s_y]} \left( 1 - p \frac{2 \mathds{E}[s_y] - 1}{1 - \mathds{E}[s_y]} \right) < \frac{1}{2 \mathds{E}[s_y]}.
    \end{align}
    Compared to a model trained using cross-entropy, we observe that the task-relevant coefficient of a model trained using mean absolute error (MAE) increases more quickly relative to the task-irrelevant coefficient. We can use induction to prove the following two properties:
    \begin{align}
        \frac{\beta^{(t)}_\text{MAE}}{\phi^{(t)}_\text{MAE}} \geq \frac{\beta^{(t)}_\text{CE}}{\phi^{(t)}_\text{CE}}\text{ and }\phi^{(t)}_\text{MAE} \leq \frac{1}{2\mathds{E}[s_y]}\phi^{(t)}_\text{CE}.
    \end{align}
    Assuming that the induction hypothesis holds at the \( t \)-th iteration, and given that the learning rate is sufficiently small, the increase in the task-irrelevant coefficient is also small. Then we can write the following:
    \begin{align}
        \frac{\beta^{(t+1)}_\text{MAE}}{\phi^{(t+1)}_\text{MAE}} = \frac{\beta^{(t)}_\text{MAE}+ \Delta\beta_\text{MAE}^{(t)}}{\phi^{(t)}_\text{MAE}+ \Delta\phi_\text{MAE}^{(t)}} \geq \frac{\beta^{(t)}_\text{CE}+\Delta \beta^{(t)}_\text{CE}}{\phi_\text{CE}^{(t)}+\Delta\phi_\text{CE}^{(t)}}  = \frac{\beta^{(t+1)}_\text{CE}}{\phi^{(t+1)}_\text{CE}}.
    \end{align}
    This inequality arises from the fact that $\frac{\beta_{\text{MAE}}^{(t)}}{\phi_{\text{MAE}}^{(t)}}\geq\frac{\beta_{\text{CE}}^{(t)}}{\phi_{\text{CE}}^{(t)}}$ and the increments satisfy $\frac{\Delta\beta_{\text{MAE}}^{(t)}}{\Delta\phi_{\text{MAE}}^{(t)}}\geq\frac{\Delta\beta_{\text{CE}}^{(t)}}{\Delta\phi_{\text{CE}}^{(t)}}$. Additionally, we have:
    \begin{align}
        \phi^{(t+1)}_\text{MAE} = \phi_\text{MAE}^{(t)} + \Delta \phi_\text{MAE}^{(t)} < \phi_\text{MAE}^{(t)} + \frac{1}{2\mathds{E}[s_y]}\Delta \phi^{(t)}_\text{CE} \leq  \frac{1}{2\mathds{E}[s_y]}\phi_\text{CE}^{(t)} + \frac{1}{2\mathds{E}[s_y]}\Delta \phi^{(t)}_\text{CE} = \frac{1}{2\mathds{E}[s_y]}\phi^{(t+1)}_\text{CE}.
    \end{align}
    In conclusion, after a sufficiently large number of iterations, we have:
    \begin{align}
        L_{\mathcal{D}} (\mathbf{p}_\text{CE}) \geq L_{\mathcal{D}}(\mathbf{p}_\text{MAE}),
    \end{align}
    which demonstrates that the test loss for mean absolute error (MAE) in prompt learning with noisy labels is lower than that of cross-entropy loss, highlighting the robustness of MAE.
\end{proof}

\end{document}